\documentclass[letterpaper, 10 pt, conference]{ieeeconf}  
\usepackage{gen_settings}

\IEEEoverridecommandlockouts                              

\title{\LARGE \bf
Verifiable Learned Behaviors via Motion Primitive Composition: \\
Applications to Scooping of Granular Media
}

\author{Andrew Benton, Eugen Solowjow, Prithvi Akella$^{1}$
\thanks{$^{1}$All authors are with Siemens Corporation
        {\tt\small \{andrew.benton, prithvi.akella, eugen.solowjow\}@siemens.com}}%
}

\begin{document}

\maketitle
\thispagestyle{empty}
\pagestyle{empty}

\begin{abstract}
A robotic behavior model that can reliably generate behaviors from natural language inputs in real time would substantially expedite the adoption of industrial robots due to enhanced system flexibility. To facilitate these efforts, we construct a framework in which learned behaviors, created by a natural language abstractor, are verifiable by construction. Leveraging recent advancements in motion primitives and probabilistic verification, we construct a natural-language behavior abstractor that generates behaviors by synthesizing a directed graph over the provided motion primitives. If these component motion primitives are constructed according to the criteria we specify, the resulting behaviors are probabilistically verifiable. We demonstrate this verifiable behavior generation capacity in both simulation on an exploration task and on hardware with a robot scooping granular media.
\end{abstract}

\section{Introduction}
In recent years, learning from human demonstrations has proven tremendously successful at imitating intricate, human-like motion on robotic systems~\cite{atkeson1997robot,ravichandar2020recent,zhu2018robot}.  This has allowed for improvements in robotic grasping~\cite{lin2012learning,pastor2009learning,misimi2018robotic}, assembly~\cite{zhu2018robot,manyar2023inverse,wang2014towards}, and even robotic surgery~\cite{keller2020optical,osa2014trajectory,kim2020autonomously}.  However, these methods often require prohibitive amounts of precisely labeled data~\cite{Hussein2017Imitation}.  Additionally, these learned behaviors are typically not transferrable to tasks that are similar but not identical, prompting further research into task-transferrable learning~\cite{chao2011towards,hausman2017multi,finn2017one}. However, works in this vein exhibit similar, if not heightened, requirements on the amount of data available to the learning procedure.

Despite these challenges, more comprehensive learned models that incorporate streams of multimodal data have shown tremendous success at learning generalized, intricate behaviors.  For example, the recently developed Palm-E model has successfully translated natural language user commands to control policies for a $6$-DOF arm, realizing the intended tasks even when they were not explicitly learned~\cite{driess2023palm}.   Building on the success of Palm-E and other foundational robotic models~\cite{cui2022can,shah2023vint,bommasani2021opportunities}, recent work also aims to codify effective design principles for these models~\cite{vemprala2023chatgpt}.

\begin{figure}[t]
    \centering
    \includegraphics[width = \columnwidth]{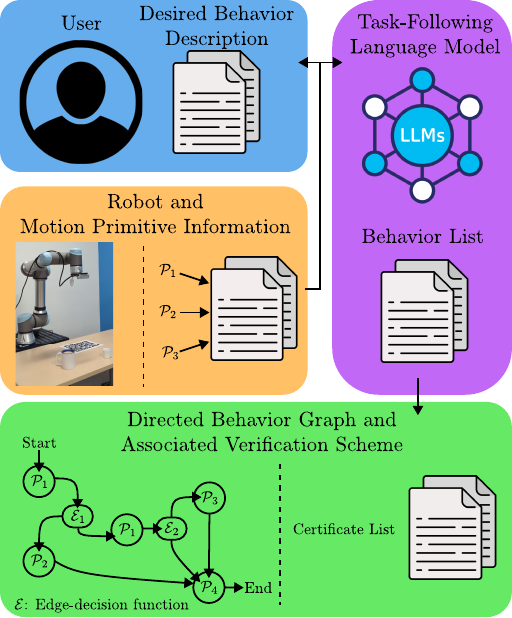}
    \caption{A graphical representation of our natural-language-based behavior generalizer and verification scheme.  By ensuring that the language model only composes behaviors as a directed graphical abstraction over the provided motion primitives, we show that any such generated behavior has an associated certificate list that we can exploit to verify the learned behavior's ability to realize the user's desired task.} \vspace{-0.25 in}
    \label{fig:title}
\end{figure}

Conceptually, however, both the Palm-E model and the other learning paradigms mentioned prior hinge on a notion of composing generalized behavior from a finite set of learned behaviors.  Prior work in controls and robotics has shown that generalizing from this initial behavior set, termed motion primitives in the existing literature, yields robust, and more importantly, verifiable generalized behavior provided the primitives and subsequent behaviors are constructed with care~\cite{stulp2011learning,ubellacker2023robust,akella2023probabilistic}.  Consequently, inspired by the previous attempts at codifying design principles for these learned models~\cite{vemprala2023chatgpt}, we posit that by leveraging these prior works in motion primitives and black-box risk-aware verification, we can synthesize verifiable learned behaviors over a provided set of carefully constructed motion primitives.

\newidea{Our Contribution:}
Leveraging recent work in risk-aware verification~\cite{akella2022sample,akella2022scenario}, we take steps towards constructing a framework for verifying learned, generalized behaviors composed from a set of motion primitives.  Specifically, if the input/output spaces of the motion primitives satisfy certain conditions that permit its verifiability, and the behavior is constructed as a directed graph over these primitives, then the resulting behavior is similarly verifiable.  We showcase this verifiability in both simulation and on hardware, focusing on exploration and reconnaissance for the former and a granular media scooping task for the latter.

\newidea{Structure:} We review black-box risk-aware verification and motion primitives in Section~\ref{sec:definitions} before formally stating the problem under study in Section~\ref{sec:problem_statement}.  Section~\ref{sec:verifying_behaviors} details our behavior generation scheme and states our main contribution regarding the verifiability of the resulting generated behaviors.  Finally, Section~\ref{sec:demonstrations} showcases our behavior generation scheme developing an exploratory behavior - Section~\ref{sec:simulation} - and a scooping motion for granular media - Section~\ref{sec:scooping}.  Both behaviors are also verified in the same sections according to the provided verification scheme.

\section{Terminology and Formal Problem Statement}
\label{sec:definitions}
\subsection{Black-Box Risk-Aware Verification}
\label{sec:verification}
The information in this section is adapted from~\cite{akella2022sample,akella2022scenario}.  Black-box risk-aware verification assumes the existence of a discrete-time controlled system of the following form, with system state $x \in \mathcal{X}$, control input $u \in \mathcal{U}$, environment state $d \in \mathcal{D}$ and potentially unknown dynamics $f$:
\begin{equation}
    \label{eq:system}
    x_{k+1} = f(x_k,u_k,d),~\forall~k=0,1,2,\dots.
\end{equation}
As verification measures the robustness of a controlled system's ability to realize a behavior of interest, work in this vein assumes the existence of a feedback controller $U:\mathcal{X} \times \mathcal{D} \to \mathcal{U}$. The system's evolution when steered by this controller $U$ will be denoted as $\Sigma$ - a function mapping an initial system and environment state to the system state evolution as prescribed by~\eqref{eq:system}, \textit{i.e.} $\Sigma(x_0,d) = \{x_0,x_1,\dots,x_K\} \in \mathcal{X}^K$ for some $K>0$.  Finally, a robustness measure $\rho$ maps this state evolution $\Sigma(x_0,d)$ and environment state $d$ to the reals, \textit{i.e.} $\rho: \mathcal{X}^K \times \mathcal{D} \to \mathbb{R}$.  For context, these robustness measures can be those coming from temporal logic~\cite{donze2010robust} or the minimum value of a control barrier function over a time horizon~\cite{ames2016control} among other methods.  A positive outcome of this robustness measure indicates that the corresponding state evolution realized the desired behavior, \textit{i.e.} $\rho(\Sigma(x_0,d),d) \geq 0$ implies the state evolution $\Sigma(x_0,d)$ realized the behavior of interest.

Black-box risk-aware verification employs this robustness measure to provide a probabilistic statement on the system's ability to realize the desired behavior for all permissible initial conditions and environment states.  This will formally be expressed in the following theorem:
\begin{theorem}
    \label{thm:verification}
    Let $\{r^i = \rho(\Sigma(x_0^i,d^i),d^i)\}_{i=1}^N$ be a set of $N$ robustness evaluations of trajectories whose initial conditions and environments $(x_0^i,d^i)$ were sampled via $\pi$ over $\mathcal{X} \times \mathcal{D}$, and let $r^* = \min \{r_1,r_2,\dots,r_N\}$.  Then, both the probability of sampling an initial condition and environment evolution pair whose robustness is lower bounded by $r^*$ and the confidence in the associated probability is only a function of the number of samples $N$ and a scalar $\epsilon \in [0,1]$, \textit{i.e.}
    \begin{equation}
        \label{eq:prob_verification}
        \prob^N_{\pi}\left[\prob_{\pi}[\rho(\Sigma(x_0,d),d) \geq r^*] \geq 1-\epsilon \right] \geq 1-(1-\epsilon)^N.
    \end{equation}
\end{theorem}

\subsection{Motion Primitives}
Motion primitives are a well-studied field in the controls and robotics literature, though we will provide a slight variant on existing definitions to align with our notation.
\begin{definition}
    \label{def:motion_primitive}
    A \textit{Motion Primitive} is $4$-tuple $\mathcal{P} = (\Xi,A,U,R)$ with the following definitions for the tuple: 
    \begin{itemize}
        \item[$(\Xi)$] The complete set of parameters for this primitive, \textit{i.e.} $\Xi \subseteq \mathbb{R}^p$ for an appropriate dimension $p\geq0$.
        \item[$(A)$] A function taking a system and environment state $(x,d)$ as per~\eqref{eq:system} and outputting the subset of valid parameters $P$ for this pair, \textit{i.e.} $A(x,d) = P \subseteq \Xi$.
        \item[$(U)$] The parameterized controller for this primitive, mapping states, environments, and the parameter to inputs, \textit{i.e.} $U:\mathcal{X} \times \mathcal{D} \times \Xi \to \mathcal{U}$.
        \item[$(R)$] A parameterized function outputting the subset of the state space the system will occupy upon completion of the primitive, \textit{i.e.} for $\xi \in \Xi$ and with environment state $d$, $R(\xi,d) = X_f \subseteq \mathcal{X}$.
    \end{itemize}
\end{definition}

As an example consistent with the simulations to follow then, consider the system as per~\eqref{eq:system} to be a single-integrator system on the plane required to navigate in a finite-sized grid.  A feasible motion primitive $\mathcal{P}$ would be moving the system to an adjacent cell.  For simplicity's sake, assume there are no obstacles, and as such, the environment state space $\mathcal{D} = \varnothing$.  Then, the complete set of parameters $\Xi$ would be the labels for all the cells in this grid, the accepting function $A$ outputs all adjacent cells to the cell containing the current system state $x$, $U$ could be a proportional controller sending the system to the appropriate grid, and $R$ would output the subset of the state space encompassed by the cell to which the system was required to move.  

\subsection{Problem Statement}
\label{sec:problem_statement}
Our goal is to develop a framework by which behaviors learned over these primitives can be verified.  As such, we define a behavior $B$ as a directed graph of primitives, with edges from a primitive $\mathcal{P}$ indicating the primitive $\mathcal{P}'$ to be run upon completion of $\mathcal{P}$. 
For examples of such behaviors, see the sketch provided in Figure~\ref{fig:title} and the resulting behavior for our simulation example in Figure~\ref{fig:sim_behavior_graph}.  The formal definition of these behaviors will follow.
\begin{definition}
    \label{def:behavior}
    A \textit{behavior} $B$ is a directed graph defined as a $4$-tuple, \textit{i.e.} $B = (N,E,S,T)$ with the following definitions:
    \begin{itemize}[align = parleft, labelsep = 3em, leftmargin = 3 em]
        \item[$(N)$] The finite set of nodes for the graph, where each node is a primitive as per Definition~\ref{def:motion_primitive}, \textit{i.e.} $N = \{\mathcal{P}_1, \mathcal{P}_2,\dots,\mathcal{P}_{|N|}\}$.
        \item[$(E)$] The set of directed edges connecting nodes in the graph.  Each edge identifies a method to choose parameters for the successive primitive.  If multiple edges emanate from a node, then a method exists such that at runtime, only one edge is chosen.
        \item[$(S)$] A start function taking as input the system and environment state $(x,d)$ as per~\eqref{eq:system} and outputting both the starting primitive and its parameter, \textit{i.e.} $S(x,d) = (\xi,\mathcal{P})$ where $\mathcal{P} \in N$ and $\xi \in A_{\mathcal{P}}(x,d)$.
        \item[$(T)$] The set of terminal nodes, \textit{i.e.} $T \subseteq N$.  
    \end{itemize}
\end{definition}

Our goals are twofold.  First, determine whether we can verify the behaviors generated by Algorithm~\ref{alg:behavior}, and second, if the behaviors are verifiable, determine a framework by which we can verify any behavior generated by this method.  Phrased formally, the problem statement will follow.
\begin{problem}
    \label{prob_statement}
    Determine if the behaviors generated by Algorithm~\ref{alg:behavior} are verifiable, and if they are verifiable, determine a method to verify any such generated behavior.
\end{problem}

\begin{algorithm}[t]
\caption{Natural Language-based Behavior Generalizer}\label{alg:behavior}
\begin{algorithmic}
\Require A set of primitives as per Definition~\ref{def:motion_primitive} and their descriptions $\mathbb{D} = \{(\mathcal{P}_i,$ description of primitive $i)\}_{i=1}^M$, a list of existing behaviors $\mathbb{B} = \{B_1,B_2,\dots\}$ with behaviors $B$ as per Definition~\ref{def:behavior}, and a natural language abstractor $A$ taking as input a string $s$ defining a desired behavior, a string $I$ defining any useful, non-primitive information available for behavior generation, and the primitive list $\mathbb{D}$ and outputting behaviors $B$.
\While{True}
\State $c \gets$ desired behavior $B$
\If{$c \not \in \mathbb{B}$}
\State $s \gets$ description of desired behavior
\State $I \gets$ helpful non-primitive information
\State $\mathbb{B} \gets \mathbb{B} \bigcup A(s,I)$
\EndIf
\EndWhile
\end{algorithmic}
\end{algorithm}

\section{Verifying Learned Behaviors}
\label{sec:verifying_behaviors}
We will provide a solution to both aspects of Problem~\ref{prob_statement} simultaneously, by constructing the framework for verifying any behavior as per Definition~\ref{def:behavior}.  To construct this framework, we first note that there exist two outcomes to executing any behavior from any initial system and environment state - it either terminates successfully or it does not.  In the event it terminates successfully, we can record the set of all primitives run over the course of the behavior, their corresponding parameters, and the system states upon termination of the corresponding primitive, \textit{i.e.} $\mathbb{D} = \{(\xi_1,\mathcal{P}_1,x^f_1), (\xi_2,\mathcal{P}_2,x^f_2),\dots\}$.  If the behavior fails due to reasons such as an intermediary controller failure or an error in the behavior's graph construction leading to a runtime error, we can record the failure.  

This permits us to construct a robustness measure for a verification scheme aligned with the method described in Section~\ref{sec:verification}.  First, for each pair in the dataset $\mathbb{D}$ generated by running the behavior, we can define a certificate function checking whether the terminal state laid in the terminal set prescribed by the primitive, parameter, and environment:
\begin{equation}
    \label{eq:certificate}
    C\left(\xi,\mathcal{P},x^f,d\right) = x^f \in R_{\mathcal{P}}(\xi,d).
\end{equation}
Here, we note that we are implicitly associating boolean outcomes with $\pm 1$.  The robustness measure $\rho$ would check the validity of each of these certificates over the run of a behavior and output $1$ if all certificates were satisfied and $-1$ if the system failed or any certificate was not satisfied.  Specifically then, let $(x_0,d)$ be the initial system and environment state, let $\Sigma$ be the trajectory function as described in Section~\ref{sec:verification}, and let $\mathbb{D}$ be the dataset of tuples collected over the course of a successfully run behavior.  Then the robustness measure
\begin{equation}
    \label{eq:behavior_robustness}
    \rho_B(\Sigma(x_0,d),d) = \begin{cases}
        \min\limits_{\gamma \in \mathbb{D}}~C(\gamma,d) & \mbox{if~behavior~finished}, \\
        -1 & \mbox{else}.
    \end{cases}
\end{equation}
Here, we have abbreviated the tuples in $\mathbb{D}$ with the variable $\gamma$ to ease notation.  That being said, the robustness measure $\rho_B$ in~\eqref{eq:behavior_robustness} evaluates to a positive number if and only if the behavior successfully terminated and all component primitives exhibited their component desired behaviors.

\begin{figure*}[t]
    \centering
    \includegraphics[width = \textwidth]{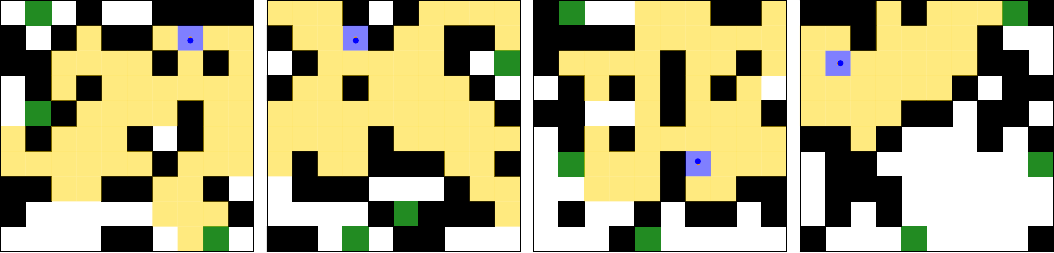}
    \caption{Examples of the environments considered for the example in Section~\ref{sec:simulation}.  The blue circle represents the agent, the blue square represents the agent's starting cell, the green squares are goals, the black squares are obstacles, and the gold region is the region explored by the learned behavior.} \vspace{-0.25 in}
    \label{fig:sim_examples}
\end{figure*}

Using the robustness measure in~\eqref{eq:behavior_robustness}, we can verify any behavior as per Definition~\ref{def:behavior}.  To ease the formal exposition of the results, we will first denote via $\mathcal{B}$ the subset of the system and environment state spaces that have a valid starting point for the behavior $B$ to be verified.  This is to ensure that in the verification procedure to follow, we do not sample and evaluate the behavior's performance from initial conditions and environment states that disallow the behavior from the start.  Formally then,
\begin{equation}
    \label{eq:allowable_space}
    \mathcal{B} = \{(x,d) \in \mathcal{X} \times \mathcal{D}~|~S_B(x,d) \neq \varnothing\}.
\end{equation}
With these definitions we have the following theorem identifying a framework to verify behaviors, though to simplify exposition, we will express the assumptions separately:
\begin{assumption}
    \label{assump:behavior_volume_fraction}
    Let $\{r^i = \rho_B(\Sigma(x^i_0,d^i),d^i)\}_{i=1}^N$ be the behavioral robustness of $N$ attempts at executing behavior $B$ from uniformly sampled initial conditions and states $(x_0,d)$ over the allowable space $\mathcal{B}$ as per~\eqref{eq:allowable_space} with robustness measure $\rho$ as per~\eqref{eq:behavior_robustness}, and let $r^* = \min_i r^i$.
\end{assumption}
\begin{theorem}
    \label{thm:behavior_volume_fraction}
    Let Assumption~\ref{assump:behavior_volume_fraction} hold.  If $r^* = 1$, then $\forall~\epsilon \in [0,1]$, the behavior $B$ will execute successfully for at least $100(1-\epsilon)\%$ of the initial condition and environment pairs in $\mathcal{B}$ and the confidence in this statement is $1-(1-\epsilon)^N$.
\end{theorem}
\begin{proof}
    As Assumption~\ref{assump:behavior_volume_fraction} satisfies the conditions for Theorem~\ref{thm:verification}, we can employ the same theorem and get the following result $\forall~\epsilon \in [0,1]$ and substituting $r^* = 1$:
    \begin{equation}
    \begin{aligned}
        \mathbb{C}1 & \triangleq \prob_{\uniform[\mathcal{B}]}[\rho_B(\Sigma(x_0,d),d) \geq 1] \geq 1-\epsilon, \\
        \mathbb{C}2 & \triangleq \prob^N_{\uniform[\mathcal{B}]}[\mathbb{C}1] \geq 1-(1-\epsilon)^N.
    \end{aligned}
    \end{equation}
    Here, $\uniform[\mathcal{B}]$ denotes the uniform distribution over $\mathcal{B}$.  We will analyze $\mathbb{C}1$ first.  Note that in order for $\rho_B(\Sigma(x_0,d),d) \geq 1$, all certificate functions over the dataset $\mathbb{D}$ generated by running behavior $B$ must evaluate to $1$ - a consequence of equations~\eqref{eq:behavior_robustness} and~\eqref{eq:certificate}.  As a result,
    \begin{equation}
        \label{eq:valid_robustness}
        \rho_B(\Sigma(x_0,d),d) = 1 \iff
        \begin{gathered}
            \mathrm{The~behavior~executes} \\
            \mathrm{successfully.}
        \end{gathered}
    \end{equation}
    Therefore, we can define a subset of the feasible joint state space, corresponding to initial conditions and environment states where from and in the behavior executes successfully:
    \begin{equation}
        \mathbb{V} = \{(x,d) \in \mathcal{B}~|~\rho(\Sigma(x,d),d) = 1 \}.
    \end{equation}
    Similarly, we can define a volume fraction function over the allowable joint state space:
    \begin{equation}
        \mathcal{V}(Q) = \frac{\int_Q 1 ds}{\int_{\mathcal{B}} 1 ds}.
    \end{equation}
    Finally, since the uniform distribution assigns probabilistic weight to a subset of events equivalent to their volume fraction in the sample space, $\mathbb{C}1$ resolves to the following:
    \begin{equation}
        \mathbb{C}1 \equiv \mathcal{V}(\mathbb{V}) \geq 1-\epsilon.
    \end{equation}
    Substituting this equivalency in $\mathbb{C}2$ completes the proof.
\end{proof}

\subsection{Extending to Non-Deterministic Behaviors}
In the prior sections, we only considered deterministic system evolution and behavior graph resolution.  However, it may be the case that either the system evolves or the behavior graph resolves non-deterministically.  Our proposed verification framework should account for this non-determinism, and this section details how the prior procedure extends to this case.  We will formalize this non-determinism directly in the context of verification.  Specifically, we assume that we have a distribution by which we can draw robustness evaluations of system trajectories, \textit{i.e.}
\begin{equation}
    \rho(\Sigma(x_0,d),d)~\mathrm{is~sampled~from~}\pi_V
\end{equation}
Note that this accounts for both cases where the initial system and environment states are potentially sampled randomly via a distribution $\pi_X$ over the allowable space $\mathcal{B}$ as per~\eqref{eq:allowable_space} and the ensuing trajectories $\Sigma(x_0,d)$ are also randomly sampled from some unknown trajectory-level distribution $\pi_S$, arising from the aforementioned non-deterministic system evolution or behavior graph resolution.

As a result, we can follow the same verification method as in Theorem~\ref{thm:verification}, though we cannot identify trajectories via initial conditions as we did in Assumption~\ref{assump:behavior_volume_fraction}.  The following assumption and corollary expresses this notion formally:
\begin{assumption}
    \label{assump:non_deterministic_verification}
    Let $\rho_B$ be the robustness measure for the behavior $B$ as per equation~\eqref{eq:behavior_robustness}, let $\{r^i = \rho_B(\Sigma(x^i_0,d^i),d^i)\}_{i=1}^N$ be the robustnesses of $N$ trajectories sampled via the (unknown) distribution $\pi_V$, and let $r^* = \min_i r^i$.
\end{assumption}
\begin{corollary}  
    \label{corr:non_deterministic}
    Let Assumption~\ref{assump:non_deterministic_verification} hold.  If $r^* = 1$, then $\forall~\epsilon \in [0,1]$, the non-deterministic system $\Sigma$ successfully executes the behavior $B$ with minimum probability $1-\epsilon$ and confidence $1-(1-\epsilon)^N$, \textit{i.e.}:
    \begin{equation}
        \prob^N_{\pi_V}\left[\prob_{\pi_V}[\rho(\Sigma(x_0,d),d) \geq r^*] \geq 1-\epsilon \right] \geq 1-(1-\epsilon)^N.
    \end{equation}
\end{corollary}
\begin{proof}
    This is a direct result of Theorem~\ref{thm:verification}.
\end{proof}


\section{Demonstrations}
\label{sec:demonstrations}
\subsection{Exploratory Behavior Generation}
\label{sec:simulation}
To illustrate the verifiability of behaviors generated via Algorithm~\ref{alg:behavior}, this section will detail our efforts at using a natural language abstractor built on ChatGPT to construct an exploratory behavior.

\spacing
\newidea{System and Environment Description:} To that end, the simulations to follow feature an agent idealized as a single integrator system on the plane and navigating within a $10 \times 10$ grid with obstacles and a few goals. The system state $x$ is its planar position and its labels for each of the cells, \textit{i.e.} $x \in [-5,5]^2 \times \{$empty, obstacle, unexplored, goal$\}^{100} \triangleq \mathcal{X}$.  The environment, \textit{i.e.} obstacle and goal cells, is the subset of the overall label space where there exist $30$ obstacles and $3$ goals with no overlaps, \textit{i.e.} $\mathcal{D} \subset$ $\{$empty, obstacle, goal$\}^{100}$.  The system dynamics as per~\eqref{eq:system} are known in this case, with single-integrator dynamics for the planar dimension and label updates when specifically provided by a controller - otherwise, labels remain constant.

\spacing
\newidea{Motion Primitives:} The system has two primitives upon which the natural-language behavior generalizer can build behaviors.  Their descriptions will follow:
\begin{itemize}
    \item[$\mathcal{P}^s_1:$] A label update function that updates the labels in the state $x$ to match the labels of the cells immediately surrounding the agent, \textit{i.e.} if the agent were in cell $(2,3)$ the function updates the labels of cells $\{(2,3),(3,3),(1,3),(2,4),(2,2)\}$.
    \begin{itemize}
        \item[$\Xi:$] The set of all cells, \textit{i.e.} $\Xi = \{0,1,2,\dots,9\}^2$.
        \item[$A:$] A function outputting the cell the system currently occupies, \textit{i.e.} if the system's planar position were $[-4.5,-3.5]$, the only valid parameter is cell $(0,1)$.
        \item[$U:$] Updates the state to reflect the environment labels of all adjacent cells.
        \item[$R:$] A function outputting the portion of the state space where the labels for the agent's current and adjacent cells align with those of the environment.  All other cell labels are unconstrained, \textit{i.e.} if the agent's current and adjacent cells were all empty, then $R(\xi,d)$ would output the subset of the state space containing label vectors whose elements for those cells all read ``empty" with no constraints on other elements. 
    \end{itemize}
    \item[$\mathcal{P}^s_2:$] A navigation function that steers the agent to a desired cell while avoiding obstacles.
    \begin{itemize}
        \item[$\Xi:$] The set of all cells, \textit{i.e.} $\Xi = \{0,1,2,\dots,9\}^2$.
        \item[$A:$] A function outputting the portion of the parameter space where the cell is reachable by the agent in the provided environment.
        \item[$U:$] A Markov-Decision-based planner tracked by a PD controller that steers the agent to the desired cell while avoiding obstacles.
        \item[$R:$] Outputs the portion of the planar state space encompassed by the desired cell, \textit{i.e.} if the agent could reach cell $(2,2)$, then $R(\xi = (2,2),d) = [-2,-1]^2$.
    \end{itemize}
\end{itemize}

\spacing
\newidea{Algorithm Information:} We desired an exploratory behavior whereby the system searches the grid for a goal and after identifying a goal, oscillates back and forth between the goal and its starting location at least $5$ times.  As useful information for the task-following algorithm, the inputted information - string $I$ in Algorithm~\ref{alg:behavior} - indicated that the language model could use the following functions when determining edges in the outputted behavior graph:
\begin{itemize}
    \item[$\mathcal{E}^s_1:$] A function that outputs as a list, all the cells that have been explored by the agent thus far, \textit{i.e.} all cells that have a label other than ``unexplored" in the current state.
    \item[$\mathcal{E}^s_2:$] A function that outputs as a list all cells immediately adjacent to the agent's currently occupied cell.
    \item[$\mathcal{E}^s_3:$] A function that determines whether a goal has been found and outputs the corresponding cell.
\end{itemize}

\begin{figure}[t]
    \centering
    \includegraphics{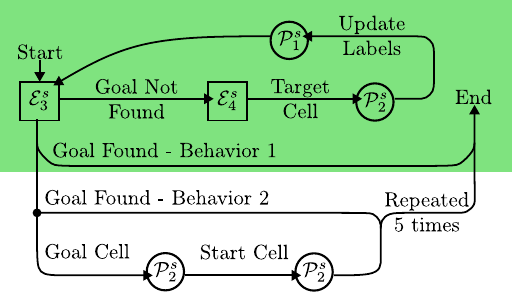}
    \caption{Depiction of the directed behavior graph generated by Algorithm~\ref{alg:behavior} for the example detailed in Section~\ref{sec:simulation}.  The first behavior's graph is highlighted in green, the second behavior incorporates the first and the extra information is the unhighlighted part of the graph.}  \vspace{-0.25 in}
    \label{fig:sim_behavior_graph}
\end{figure}

\begin{figure*}[t]
    \centering
    \includegraphics[width = \textwidth]{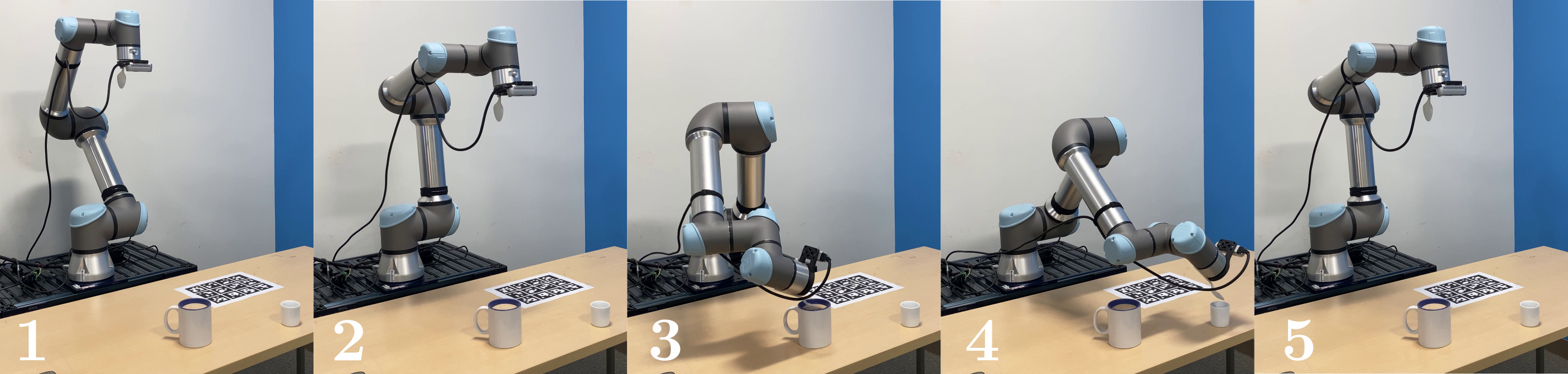}
    \caption{Depiction of the learned scooping behavior.  In this case, the motion was coded previously, but contingent on the arm's ability to sense the cups in its environment.  As such, the LLM interface only asked for the end-user to provide that initial positioning (1) wherein the arm had a high likelihood of sending both cups. Then the LLM behavior first moves to the desired sensing position (2), calls the scooping primitive as seen in (3)-(4), and returns to the instructed sensing position in (5) in case any of the cups shifted during the procedure.  Then the process repeats.} \vspace{-0.25 in}
    \label{fig:arm_tiles}
\end{figure*}

\spacing
\textit{Behavior 1:} For the first step, we asked the algorithm to devise a behavior that explored the grid until it identified a goal.  Specifically, the inputted behavior string $s$ was as follows: \textit{``Please construct a function that performs the following tasks in sequence.  First, it searches over all explored cells that are not obstacles to find the explored cell with the highest number of unexplored neighbors.  Let's call this identified cell, cell A.  Second, it sends the agent to cell A and identifies the labels of all adjacent cells.  Three, it repeats steps one and two until a goal has been found, at which point, it stops."}  The part of the graph highlighted in green in Figure~\ref{fig:sim_behavior_graph} shows the generated behavior graph.  As part of this generation procedure, it used two of the provided functions $\mathcal{E}^s_1,\mathcal{E}^s_2$ to construct the edge decision function $\mathcal{E}^s_4$ whose description will follow:
\begin{itemize}
    \item[$\mathcal{E}^s_4:$] A function that searches over all explored cells - the list of explored cells is provided by $\mathcal{E}^s_1$ - and assigns to each cell the number of its adjacent cells that are unexplored - the list of adjacent cells is provided by $\mathcal{E}^s_2$.  Reports the first cell in the list with the maximum number of unexplored neighbors.
\end{itemize}

\spacing
\textit{Behavior 2:} We wanted to build on the prior behavior for the latter half of our goal, and as such, informed the LLM of the existence of this prior behavior in the list of existing behaviors denoted as $\mathbb{B}$ in Algorithm~\ref{alg:behavior}.  Then, as the user, we requested the following from the LLM: \textit{``Please construct a function that performs the following tasks in sequence.  First, it finds a goal.  Second, it moves between the goal and its starting location 5 times."} 
The behavior graph for this second behavior is the unhighlighted graph in Figure~\ref{fig:sim_behavior_graph}.

\spacing
\newidea{Verification Procedure and Remarks:}  As Behavior $2$ utilized Behavior $1$, verifying both amounts to verifying the former.  Following the results of Theorem~\ref{thm:behavior_volume_fraction}, we recorded a data set $\mathbb{D}$ of parameters, primitives, and terminal states while running the second behavior.  The certificates per equation~\eqref{eq:certificate} amount to checking that updated labels matched their true labels after running primitive $\mathcal{P}^s_1$ and checking that the system occupied the desired cell after running primitive $\mathcal{P}^s_2$.  The allowable joint state space $\mathcal{B}$ as per~\eqref{eq:allowable_space} was the portion of the joint space where the system starts in a state $x$ such that at least one goal is reachable in the corresponding environment $d$.  Finally, the verification procedure uniformly randomly sampled state pairs $(x,d) \in \mathcal{B}$ and checked the corresponding certificates for each run of the behavior.

After running the second behavior from $100$ randomly sampled initial state pairs, the behavior terminated successfully every time.  As such, by Theorem~\ref{thm:behavior_volume_fraction} we expect that the second behavior will run successfully for $95\%$ of possible state pairs and we are $99.4\%$ confident in this statement - we generated these numbers by substituting $\epsilon = 0.05$ and $N=100$ for Theorem~\ref{thm:behavior_volume_fraction}.  To validate these statements, we ran the second behavior in $2000$ more sampled environments, and it terminated successfully every time.  If we were incorrect in our prior statement that the behavior would run successfully for at least $95\%$ of feasible state pairs $(x,d) \in \mathcal{B}$, then we would have been effectively guaranteed to identify such a failure over the $2000$ subsequent runs.  As we did not, we are confident in our corresponding statement.  Furthermore, while the synthesized behaviors seem rudimentary, they suffice to indicate that our behavior synthesis scheme produces effective and verifiable behaviors.

\subsection{Scooping of Granular Media}
\label{sec:scooping}
Our second demonstration focusing on granular media scooping illustrates the framework's utility in helping end-users set up repetitive, verifiable tasks.

\spacing
\newidea{System and Environment Description:} The scooping problem consists of picking up material from a donor container and depositing it into a receiver container using a UR5e $6$-DOF robot arm with a wrist-mounted RealSense depth camera.  While a rudimentary scooping motion has been programmed \textit{apriori}, it does not know the environment in which it will be performing this motion - similar to the situation when a pre-programmed robot has to be initialized for specific use.  The robot's state $x \in \mathbb{R}^6$ is the full pose of the end-effector, the control input $u$ corresponds to joint rotations, and the environment $d$ corresponds to the locations and orientations of the donor and receiver containers and the level and distribution of sand in the donor container.

\spacing
\newidea{Motion Primitives:} In this case, the system only has one primitive, the scooping primitive, described as follows:
\begin{itemize}
    \item[$\mathcal{P}^r:$] A primitive performing a scooping motion from a donor container to a receiver container.
        \begin{itemize}
            \item[$\Xi:$] The space of feasible end-effector poses where a parameter $\xi \in \Xi$ denotes the pose in which the robot will sense all objects in the environment to start the scooping motion.
            \item[$A:$] A function outputting the space of end-effector poses from which all containers are in view of the onboard vision system.
            \item[$U:$] A controller that performs the scooping motion.
            \item[$R:$] A function that outputs a ball around the provided parameter within which the end-effector's pose will lie upon the termination of the scooping motion.
        \end{itemize}
\end{itemize}
That being said, the acceptance function $A$ is implicitly defined and impossible to know \textit{apriori}.  Here, we intend for the algorithm to assist the end-user in selecting a parameter $\xi$ whose validity, \textit{i.e.} existence in $A(x,d)~\forall~(x,d) \in \mathcal{X} \times \mathcal{D}$, can be checked through the ensuing verification procedure.

\spacing
\newidea{Algorithm Information:}  To assist the user in picking such a parameter $\xi$, the algorithm was provided an information string $I$ describing a helper function $\mathcal{E}^r_1$ that translated and rotated the end-effector a desired amount.  This string also included several examples of natural language translations to inputs for this function $\mathcal{E}^r_1$.  Additionally, the string included another function $\mathcal{E}^r_2$ that saved the end-effector pose for future reference, and the LLM was told to call this function if the user deemed the current end-effector pose satisfactory.

\spacing
\newidea{Behavior Generation and Verification:} The task-model repeatedly queried the user for end-effector translations and rotations and as to whether or not the user deemed the current pose sufficient for sensing any placement of containers.  As such, there was no singular behavior prompt $s$.  However, as the resulting behavior repetitively executes the scooping primitive with the user-provided sensing pose parameter $\xi$, this behavior can be verified by the results of Corollary~\ref{corr:non_deterministic}.  To do so, before every scooping motion, we placed the containers at a computer-generated randomly chosen distance from a pre-determined set-point.  As we are manually placing containers at the pre-determined locations, there will be noise affecting this placement, though we assume this noise is independent for successive placements.  We will denote this distribution of container placements via $\pi$.  As there is no need to sample over initial robot states - the system always starts and ends at the parameterized sensing pose $\xi$ every iteration - we can draw independent environments - container placements - via our distribution $\pi$ and record the robot's ability to perform its scooping motion in each placement.  Doing so for $59$ sampled environments with successful trials each time indicates according to Corollary~\ref{corr:non_deterministic} that if we continued to sample environments and test the system accordingly, the system would succeed at least $95\%$ of the time and we are at least $95\%$ confident in that statement.

\section{Conclusion}
We propose a framework by which a natural language abstractor can synthesize verifiable behaviors as a directed graph over provided motion primitives.  To showcase the increased flexibility and verifiability of the synthesized behaviors, we instructed the task-following model to construct an exploratory behavior for a simulated planar agent and a scooping behavior for a robotic arm.  In both cases, the generated behavior was verifiable via the aforementioned method, and we were able to validate our probabilistic verification statements in simulation.

\bibliographystyle{IEEEtran}
\balance
\bibliography{IEEEabrv,bib_works}

\end{document}